\documentclass[letterpaper]{article}
\usepackage{uai2018}
\usepackage{amsmath,amsthm,verbatim,amssymb,amsfonts,amscd, graphicx, mathrsfs,
tikz-cd, tikz, palatino, tikz-cd, graphics,
subcaption, booktabs, pgfplots, pst-3dplot, float, 
}

\usepackage{natbib}
\usepackage{enumitem}
\usepackage[]{hyperref}
\usetikzlibrary{shapes.misc, positioning, shapes.geometric, backgrounds, patterns}
\tikzset{cross/.style={cross out, draw=black, minimum size=2*(#1-\pgflinewidth), inner sep=0pt, outer sep=0pt},
cross/.default={2pt}}

\setlist[enumerate]{noitemsep, topsep=0pt}
\usepackage[ruled, linesnumbered,noend]{algorithm2e}

\SetCommentSty{mycommfont}
\usepackage[margin=1in]{geometry}

\theoremstyle{plain}
\newtheorem{theorem}{Theorem}
\numberwithin{theorem}{section}

\numberwithin{corollary}{section}
\newtheorem{lemma}{Lemma}
\numberwithin{lemma}{section}
\newtheorem{proposition}{Proposition}
\numberwithin{proposition}{section}
\theoremstyle{definition}
\newtheorem{definition}{Definition}
\numberwithin{definition}{section} 

\numberwithin{example}{section} 
\theoremstyle{remark}

\numberwithin{remark}{section} 

\def\Pr{\mathop{\rm Pr}\nolimits}
\usepackage{times}

\title{Generating and Sampling Orbits for Lifted Probabilistic Inference}
\author{ \textbf{Steven Holtzen}
  \and \textbf{Todd Millstein}
  \and \textbf{Guy Van den Broeck}
  \AND Computer Science Department \\ University of California, Los Angeles \\
  \texttt{\{sholtzen,todd,guyvdb\}@cs.ucla.edu}
}

\tikzset{
  annotated cuboid/.pic={
    \tikzset{%
      every edge quotes/.append style={midway, auto},
      /cuboid/.cd,
      #1
    }
    \draw [every edge/.append style={pic actions, densely dashed, opacity=.5}, pic actions]
    (0,0,0) coordinate (o) -- ++(-\cubescale*\cubex,0,0) coordinate (a) -- ++(0,-\cubescale*\cubey,0) coordinate (b) edge coordinate [pos=1] (g) ++(0,0,-\cubescale*\cubez)  -- ++(\cubescale*\cubex,0,0) coordinate (c) -- cycle
    (o) -- ++(0,0,-\cubescale*\cubez) coordinate (d) -- ++(0,-\cubescale*\cubey,0) coordinate (e) edge (g) -- (c) -- cycle
    (o) -- (a) -- ++(0,0,-\cubescale*\cubez) coordinate (f) edge (g) -- (d) -- cycle;
  },
  /cuboid/.search also={/tikz},
  /cuboid/.cd,
  width/.store in=\cubex,
  height/.store in=\cubey,
  depth/.store in=\cubez,
  units/.store in=\cubeunits,
  scale/.store in=\cubescale,
  width=5,
  height=5,
  depth=5,
  units=cm,
  scale=.1,
}

\newcommand{\orb}[0]{\mathrm{Orb}}
\newcommand{\fix}[0]{\mathrm{Fix}}

\newcommand{\stab}[0]{\mathrm{Stab}}
\newcommand{\varX}[0]{\mathbf{X}}
\newcommand{\varx}[0]{\mathbf{x}}

\newcommand{\evidence}[0]{\mathbf{e}}
\newcommand{\bool}[0]{\mathbb{B}}

\newcommand{\group}[0]{\mathcal{G}}

\newcommand{\col}[0]{\mathtt{color}}
\newcommand{\graph}[0]{\mathtt{G}}
\newcommand{\fgraph}[0]{\mathcal{F}}

\newcommand{\aut}[0]{\mathbb{A}}

\newcommand{\true}[0]{\mathtt{T}}
\newcommand{\false}[0]{\mathtt{F}}
\newcommand{\poly}[0]{\mathrm{poly}}

\DeclareMathOperator*{\argmax}{arg\,max}

\tikzstyle{bdd}=[
  >=latex, thick, >=stealth, font=\small,auto,scale=0.9,every node/.style={scale=0.9}
]
\tikzstyle{bddnode}=[
  draw,circle,inner sep=0pt,fill=white,minimum size=5.5mm
]

\tikzstyle{highedge}=[
    line width=0.9
]
\tikzstyle{lowedge}=[
    line width=0.9,dotted
]
\tikzstyle{bddterminal}=[
  draw,fill=gray!20,inner sep=2.5pt, font=\small
]

\begin{document}
\maketitle
\begin{abstract}
  A key goal in the design of probabilistic inference algorithms is identifying
and exploiting properties of the distribution that make inference tractable.
Lifted inference algorithms identify \emph{symmetry} as a property that enables
efficient inference and seek to scale with the degree of symmetry of a
probability model. A limitation of existing exact lifted inference techniques is
that they do not apply to non-relational representations like factor graphs. In
this work we provide the first example of an exact lifted inference algorithm
for arbitrary discrete factor graphs. In addition we describe a lifted
Markov-Chain Monte-Carlo algorithm that provably mixes rapidly in the degree
of symmetry of the distribution.
\end{abstract}

\section{INTRODUCTION}
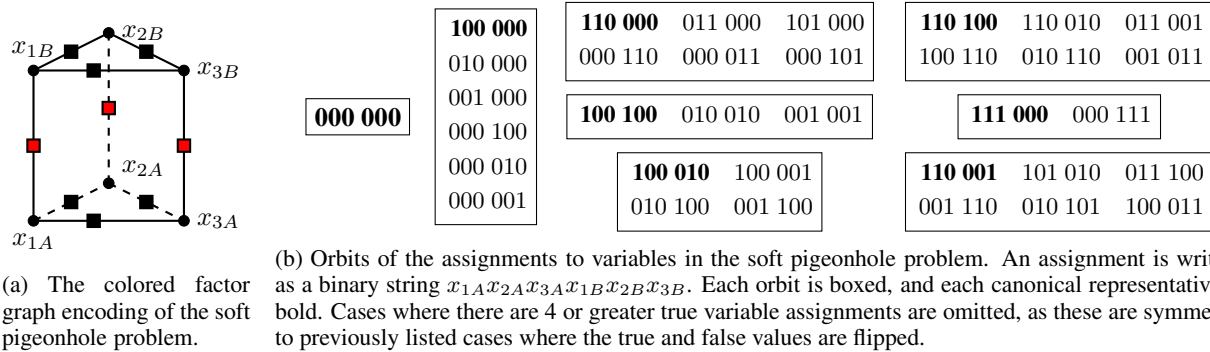
\begin{figure*}[ht]
  \begin{subfigure}[b]{0.2\linewidth}
   \centering
   
 \begin{tikzpicture}[factor node/.style={rectangle,draw,thick,solid, scale=0.7}]

   \node[factor node, fill=red] at (0, 1.5) {};
   \node[factor node, fill=red] at (-1, 1) {};
   \node[factor node, fill=red] at (1, 1) {};

   \node[factor node, fill=black] at (-0.5, 0.25) {};
   \node[factor node, fill=black] at (0.5, 0.25) {};
   \node[factor node, fill=black] at (-0.2, 0) {};

   \node[factor node, fill=black] at (-0.5, 2.25) {};
   \node[factor node, fill=black] at (0.5, 2.25) {};
   \node[factor node, fill=black] at (-0.2, 2) {};

   \begin{scope}[on background layer]
     \draw[dashed, thick] (0, 0.5) -- (0, 2.5);
     \draw[thick] (1, 0) -- (1, 2);
     \draw[thick] (-1, 0) -- (-1, 2);

     \draw[dashed, thick] (-1,0) -- (0,0.5);
     \draw[dashed, thick] (0,0.5) -- (1,0);
     \draw[thick] (-1, 0) -- (1, 0);

     \draw[thick] (-1,2) -- (0,2.5);
     \draw[thick] (0,2.5) -- (1,2);
     \draw[thick] (-1, 2) -- (1, 2);
    
    \draw[fill=black] (-1, 0) circle (2pt) node[below=0.5mm] {$x_{1A}$};
    \draw[fill=black] (0, 0.5) circle (2pt)  node[right=0.5mm,yshift=0.2cm] {$x_{2A}$};
    \draw[fill=black] (1, 2) circle (2pt) node[right=0.5mm] {$x_{3B}$};
    \draw[fill=black] (1, 0) circle (2pt) node[right=0.5mm] {$x_{3A}$};
    \draw[fill=black] (0, 2.5) circle (2pt) node[right=0.5mm] {$x_{2B}$};
    \draw[fill=black] (-1, 2) circle (2pt) node[above=0.5mm] {$x_{1B}$};
   \end{scope}
\end{tikzpicture}
\caption{The colored factor graph encoding of the soft pigeonhole problem.}
\label{fig:pigeongraph}
   
 \end{subfigure}~~~
 \begin{subfigure}[b]{0.78\linewidth}
   \centering
    \begin{tikzpicture}[main node/.style={rectangle,draw}]
      \node[main node] at(0,0) (0) {$\textbf{000~000}$};
      \node[main node] (1) [right of=0, xshift=2em] {
        ${\small
          \begin{aligned}
          \textbf{100~000}\\
          010~000\\
          001~000\\
          000~100\\
          000~010\\
          000~001\\
        \end{aligned}}$
    };

    \node[main node] (22) [right of=1, xshift=6em, yshift=0em, ] {
        ${\small
          \begin{aligned}
          \textbf{100~100} \quad 010~010 \quad 001~001
        \end{aligned}}$
    };

    \node[main node] (21) [above of=22] {
        ${\small
          \begin{aligned}
          \textbf{110~000}\quad
          011~000\quad
          101~000\\
          000~110\quad
          000~011\quad
          000~101\\
        \end{aligned}}$
    };

    \node[main node] (23) [below of=22] {
        ${\small
          \begin{aligned}
          \textbf{100~010} \quad 100~001\\
          010~100 \quad 001~100
        \end{aligned}}$
    };

    \node[main node] (31) [right of=22, xshift=10em, yshift=0em] {
        ${\small
          \begin{aligned}
          \textbf{111~000} \quad 000~111
        \end{aligned}}$
    };

    \node[main node] (32) [above of=31] {
        ${\small
          \begin{aligned}
          \textbf{110~100} \quad 110~010 \quad 011~001\\
          100~110 \quad 010~110 \quad 001~011
        \end{aligned}}$
    };

    \node[main node] (32) [below of=31] {
        ${\small
          \begin{aligned}
            \textbf{110~001} \quad 101~010 \quad 011~100 \\
            001~110 \quad 010~101 \quad 100~011
        \end{aligned}}$
    };

   \end{tikzpicture}
   \caption{Orbits of the assignments to variables in the soft pigeonhole problem. An
     assignment is written as a binary string
     $x_{1A}x_{2A}x_{3A}x_{1B}x_{2B}x_{3B}$. Each orbit is boxed, and each canonical
     representative is bold. 
     Cases where there are 4 or greater true variable assignments are omitted, as these
     are symmetric to previously listed cases where the true and false values
     are flipped.
   }
   \label{fig:pigeonorbit}
  
 \end{subfigure}
 \caption{A graphical model representation and orbit structure of the pigeonhole
   example with 3 pigeons and 2 holes.
}
  \label{fig:pigeon}
\end{figure*}

Probabilistic inference is fundamentally computationally hard in
the worst case~\citep{Roth1996}. Thus, designers of probabilistic inference
algorithms focus on identifying and exploiting sufficient conditions of the
distribution that ensure tractable inference. For instance, many existing
probabilistic inference strategies for graphical models exploit independence in
order to scale efficiently \citep{Koller2009, Darwiche2009}. The performance of
these algorithms is worst-case exponential in a graph metric known as the
\emph{treewidth} that quantifies the degree of independence in the graph.

\emph{Lifted inference} algorithms identify \emph{symmetry} as a key property
that enables efficient inference~\citep{Poole2003, Kersting2012, NiepertAAAI14}.
These methods identify \emph{orbits} of the distribution:
sets of points in the probability space that are guaranteed to have the same
probability. This enables inference strategies that scale in the number of
distinct
orbits. Highly symmetric distributions have few orbits relative to the size of
their state space, allowing lifted inference algorithms to scale to large
probability distributions with scant independence. Thus, lifted inference
algorithms identify symmetry as a complement to independence in the search for
efficient inference algorithms.

An important challenge in designing lifted inference algorithms is identifying
symmetries of a probability distribution from its high-level description. Existing exact lifted inference
algorithms rely on relational structure to extract symmetries, and thus cannot
be directly applied to propositional probability models like factor
graphs~\citep{Getoor2007}. Several approximate lifted inference algorithms ease this requirement by
extracting symmetries of the probability distribution by computing an automorphism
group of a graph, and can thus be applied directly to factor
graphs~\citep{Kersting2009, Niepert2012, Niepert2013, Bui2013}. However,
existing lifted MCMC algorithms are not guaranteed to mix
rapidly in the number of orbits.

This paper presents exact and approximate lifted inference algorithms for
arbitrary factor graphs that provably scale with the number of orbits of the
probability distribution. Inspired by the success of existing approximate lifted
inference techniques on graphical models, we apply graph isomorphism tools to
extract the necessary symmetries. First, we present a motivating example that highlights the
strengths and weaknesses of our approach. Then, we describe our exact inference
procedure. Computationally, our method combines efficient group theory
libraries like $\mathtt{GAP}$~\citep{GAP4} with graph isomorphism tools.

Next, we describe an approximate inference algorithm called \emph{orbit-jump
MCMC} that provably mixes quickly in the number of orbits of the distribution.
Orbit-jump MCMC provides an alternative to lifted MCMC \citep{Niepert2012,
Niepert2013}, a family of approximate lifted inference algorithms that compute a
single graph automorphism in order to quickly transition \emph{within} orbits. A key advantage
of lifted MCMC is that transitions do not each require a
call to a graph isomorphism tool. However, lifted MCMC relies on Gibbs sampling
to jump \emph{between orbits}, and as a consequence has no guarantees about its
mixing time in terms of the number of orbits. We will show that orbit-jump MCMC
mixes rapidly in the number of distinct orbits, at the cost of
requiring multiple graph isomorphism calls for each transition.

Note, however, that purely scaling in the number of orbits is not a panacea. Our
methods are both limited: there are liftable probability models
that still have too many orbits for our methods to be
effective. The presented methods \emph{only} exploit
symmetry,
which is in contrast to existing exact lifted inference algorithms that
simultaneously exploit
symmetry and independence.
Therefore, our algorithms scale exponentially for certain well-known liftable
distributions, such as the friends and smokers Markov logic
network~\citep{NiepertAAAI14}. Thus, we view this work
as providing a foundation for future work on inference for factor graphs that
exploits both symmetry and independence.


\section{MOTIVATION}
\label{sec:motivation}

As a motivating example, consider performing exact lifted probabilistic
inference on a probabilistic version of the pigeonhole problem. The pigeonhole
problem is a well-studied problem from automated reasoning that exhibits nuanced
symmetry. While seemingly simple, the pigeonhole problem is in fact extremely
challenging to reason about, and is often used as a benchmark in automated
reasoning tasks~\citep{Benhamou1994, Sabharwal2006, Raz2004}. A \emph{weighted
set of clauses} is a set of pairs $\Delta = \{(w, f)\}$ where $f$ is a Boolean
clause and $w \in \mathbb{R}$ is a weight. A weighted set of clauses defines a
probability distribution over assignments $\varx$ of variables in $\Delta$
according to the following:
\begin{align*}
  \Pr(\varx) = \frac{1}{Z}\exp \left[  \sum_{\{(w, f)  \in \Delta ~\mid~ \varx \models f\}} w \right],
\end{align*}
where $Z$ is a normalizing constant, and $\varx \models f$ denotes that clause
$f$ is satisfied in world $\varx$. Our goal is to compute $Z$, a task that is
\#P-hard in general.

Consider a set of weighted clauses for a \emph{soft pigeonhole problem}. There
are $n$ pigeons and $m$ holes. Each pigeon can occupy at most one hole, and
pigeons prefer to be solitary. To encode this situation as a weighted set of
clauses, for each pigeon $i$ and hole $j$, let $x_{ij}$ be a Boolean variable
that is true if and only if pigeon $i$ occupies hole~$j$.

The set $\Delta$ is a union of two sets of weighted clauses. For each
pigeon, we introduce a clause that forces it to occupy at most a single
hole:
\begin{align}
  \left( \infty,~ \overline{x_{ik}} \lor \overline{x_{il}} \right) 
  ~ \text{for each pigeon $i$ and holes $k \ne l$}.
  \label{eq:pigeonfactor}
\end{align}
An infinite weight encodes a \emph{hard clause} that must hold in the
distribution \citep{Richardson2006}. Then, for each hole we introduce
clauses that assign a positive weight to not having multiple pigeons:
\begin{align}
  \left( 2,~ \overline{x_{kj}} \lor \overline{x_{lj}} \right) 
  ~ \text{for each hole $j$ and pigeons $k \ne l$}.
  \label{eq:holefactor}
\end{align}

Figure~\ref{fig:pigeongraph} depicts this probability distribution with 3
pigeons and 2 holes as a pairwise colored factor graph,
where each weighted clause is a factor (box) and each distinct factor is given its own color
\citep{Niepert2012, Bui2013}. 
The factors in Equation~\ref{eq:pigeonfactor} are colored red, and the factors
in Equation~\ref{eq:holefactor} are colored black.

The symmetries of this probability distribution directly correspond to
automorphisms of the colored graph in Figure~\ref{fig:pigeongraph}
\citep{Bui2013, Niepert2012}.
In this paper, we consider only symmetries on assignments that arise from
symmetries on the variables.
Any permutation of vertices that preserves the
graph structure leaves the distribution unchanged.\footnote{We assume here w.l.o.g.\ but for simplicity that the factors are individually fully
symmetric. Asymmetric factors can either be made symmetric by duplicating variable nodes
\citep{Niepert2012} or encoded using colored edges \citep{Bui2013}.} Two
assignments that are
reachable from one another via a sequence of permutations are in the same
\emph{orbit}; all assignments in the same orbit thus have the same probability.
Figure~\ref{fig:pigeonorbit} shows the orbits of the 3-pigeon 2-hole 
scenario up to inversion of true and false assignments. Each orbit is boxed.
There are few orbits relative to the number of states, which is the
property that our lifted inference algorithms exploit.

We present both exact and approximate inference strategies that scale with the
number of orbits of a probability distribution. Our exact inference algorithm is
as follows. First, generate a single \emph{canonical representative} from each
orbit; in Figure~\ref{fig:pigeonorbit}, canonical representatives are shown in
bold. Then for each representative, compute the size of its orbit. If both of these
steps are efficient, then this inference computation scales efficiently with the
number of orbits, and we call it lifted. This \emph{orbit generation} procedure
is at the heart of many existing lifted inference algorithms that construct
sufficient statistics of the distribution from a relational representation
\citep{NiepertAAAI14}. We present an exact lifted inference algorithm in Section~\ref{sec:genorb} that
applies this methodology to arbitrary factor graphs by using graph isomorphism
tools to generate canonical representatives and compute orbit sizes.

Next, in Section~\ref{sec:orbsamp} we describe an approximate inference algorithm called \emph{orbit-jump MCMC}
that provably mixes quickly in the number of distinct orbits of the
distribution. This algorithm uses as its proposal the \emph{uniform orbit
distribution}: the distribution defined by choosing an orbit of the distribution
uniformly at random, and then choosing an element within that orbit uniformly at
random. We present a novel application of the \emph{Burnside process} in order
to draw samples from the uniform orbit distribution~\citep{Jerrum1992}, and show
how to implement the Burnside process on factor graphs by using graph isomorphism tools. Thus,
this orbit-jump MCMC provides an alternative to lifted MCMC that trades
computation time for provably good sample quality.

\section{BACKGROUND}
This section gives a brief description of important concepts from group theory and
approximate lifted inference that will be used throughout the
paper.\footnote{See Appendix~\ref{sec:notation} for a summary of the notation.}

\subsection{GROUP THEORY}
We review some standard terminology and notation from group theory, following \citet{Artin1998}.
A \emph{group} $\group$ is a pair $(S, \cdot)$ where $S$ is a
set and $\cdot: S \times S \rightarrow S$ is a binary
associative function such that there is an identity element and every element in
$S$ has an inverse under $(\cdot)$. The \emph{order} of a group is the
number of elements of its underlying set, and is denoted $|\group|$. A
\emph{permutation group acting on a set $\Omega$} is a set of bijections $g :
\Omega \rightarrow \Omega$ that forms a group under function composition. For
$\group$ acting on $\Omega$, a function $f: \Omega \rightarrow \Omega'$ is
\emph{$\group$-invariant} if $f(g \cdot x) = f(x)$ for any $g \in \group, x \in
\Omega$.
Two elements $x, x' \in \Omega$ are in the same \emph{orbit} under
$\group$ if there exists $g \in \group$ such that $x = g \cdot x'$. Orbit
membership is an equivalence relation, written $x \sim_\group x'$. The
set of all elements in the same orbit is denoted $\orb_\group(x)$. A \emph{stabilizer}
of $x$ is an element $g \in \group$ such that $g \cdot x =
x$; the set of all stabilizers of $x$ is a group called the
\emph{stabilizer subgroup}, denoted $\stab_\group(x)$. The subscript
in the previous notation is elided when clear. A \emph{cycle} $(x_1
~ x_2 ~ \cdots ~ x_n)$ is a permutation $x_1 \mapsto x_2, x_2 \mapsto x_3,
\cdots, x_n \mapsto x_1$. A permutation can be written as a product of
disjoint cycles.

\subsection{LIFTED PROBABILISTIC INFERENCE \& GRAPH AUTOMORPHISMS}
Lifted inference relies on the ability to identify the symmetries of probability
distributions. In existing  exact lifted inference methods, the symmetries are evident from the
relational structure of the probability model \citep{Poole2003,braz2005lifted,gogate2011probabilistic,VdBThesis13}.
In order to extend the insights of lifted inference to models where the
symmetries are less accessible, many lifted approximation algorithms rely on
graph isomorphism tools to identify the symmetries of
probability distributions \citep{Niepert2012, Bui2013, Mckay2014}.

A \emph{colored graph} is a 3-tuple $\graph = (V,E,C)$ where $(V,E)$ are the
vertices and edges of an undirected graph and $C = \{V_i\}^k_{i=1}$ is a
partition of vertices into $k$ sets. As notation, for a vertex $v$, let $\col(v,C)
= i$ if $v \in V_i$. A colored graph automorphism is an edge and
color-preserving vertex automorphism:
\begin{definition}
  Let $\graph = (V, E, C)$ and $\graph' = (V, E', C')$ be colored graphs. Then
$\graph$ and $\graph'$ are \emph{color-automorphic} to one another, denoted
$\cong$, if there exists a bijection $\phi: V \rightarrow V$ such that
\begin{enumerate}[noitemsep]
\item Vertex neighborhoods are preserved, i.e. for any $v_1, v_2 \in V$, $(v_1,
v_2) \in E \Leftrightarrow (\phi(v_1), \phi(v_2)) \in E'$;
\item Vertex colors are preserved, i.e. for all $v \in V$, $\col(v,C) =
\col(\phi(v),C')$.
\end{enumerate}
\end{definition}

The \emph{color automorphism group} of a colored graph $\graph$, denoted
$\aut(\graph)$, is the group formed by the set of color automorphisms of
$\graph$ under composition. The group $\aut(\graph)$ acts on the vertices of
$\graph$ by permuting them. Tools like $\mathtt{Nauty}$ can compute
$\aut(\graph)$ and are typically efficient in
$|V|$~\citep{Mckay2014}.

Colored graph automorphism groups are related to factor graphs via the following:
\begin{definition}
  Let $\fgraph = (\varX, F)$ be a factor graph with variables $\varX$, and
factors $F$, where $F$ are symmetric functions on assignments to variables
$\varX$, written $\varx$. Then the \emph{colored graph induced by $\fgraph$} is a tuple
$(V, E, C)$ where $V = \varX \cup F$, the set of edges $E$ connects
variables and factors in $\fgraph$, and $C$ is a partition such that (1)
factor nodes are given the
same color iff they are identical factors, and (2)
variables are colored with a single color that is distinct from the factor colors.
\end{definition}
This definition is due to \citet{Bui2013}, where the following theorem is
proved (with different terminology):
\begin{theorem}[\citet{Bui2013}, Theorem 2]
  Let $\fgraph$ be a factor graph and $\graph$ be its induced colored graph.
Then, the distribution of $\fgraph$ is $\aut(\graph)$-invariant.
\end{theorem}



\section{EXACT LIFTED INFERENCE}
\label{sec:genorb}
In this section we describe our exact lifted inference procedure. First we
discuss the group-theoretic properties of orbit generation that enable
efficient exact lifted inference. Then, we describe our algorithm for
implementing orbit generation on colored factor graphs. Finally, we present some
case studies demonstrating the performance of our algorithm.

\subsection{$\group$-INVARIANCE AND TRACTABILITY}
In this section we describe the group-theoretic underpinnings of our
orbit-generation procedure and describe its relationship with previous work on
tractability through exchangeability.
We will capture the behavior of a $\group$-invariant probability distribution on
a set of \emph{canonical representatives} of each orbit:
\begin{definition}[]
  Let $\group$ be a group that acts on a set $\Omega$. Then, there exists a
set of \emph{canonical
representatives} set $\Omega / \group \subseteq \Omega$ and surjective
\emph{canonization function} $\sigma : \Omega \rightarrow
\Omega / \group$ such that for any $x, y \in \Omega$, (1) $\orb(x) =
\orb(\sigma(x))$; and (2) $\orb(x) = \orb(y)$ if and
only if $\sigma(x) = \sigma(y)$.
\end{definition}

In statistics, $\sigma$ is often called a \emph{sufficient statistic} of a
partially exchangeable distribution \citep{NiepertAAAI14, Diaconis1980}.
The motivating example hinted at a general-purpose solution for exact inference
that proceeds in two phases. First, one constructs a representative of each
orbit; then, one efficiently computes the size of that orbit. We can formalize
this using group theory:
\begin{theorem}
  Let $\Pr$ be a $\group$-invariant distribution on $\Omega$, and evidence
$\evidence : \Omega \rightarrow \bool$ be a $\group$-invariant function. Then,
the complexity of computing the most probable explanation (MPE) is $\poly(|\Omega /
\group|)$ if the following can be computed in $\poly(|\Omega /
\group|)$:
  \begin{enumerate}[noitemsep]
  \item Evaluate $\Pr(x)$ for $x \in \Omega$;
  \item (Canonical generation) Generate a set of canonical representatives
    $\Omega / \group$,
  \end{enumerate}
 Moreover, if $|\orb(x)|$ can be computed in $\poly(|\Omega /
 \group|)$, then $\Pr(\evidence)$ can be computed in 
$\poly(|\Omega / \group|)$.
  \label{thm:tractability}
\end{theorem}
\begin{proof}
  To compute the MPE, choose:
  \begin{align}
    \argmax_{\{x \in \Omega / \group ~\mid~ \evidence(x) = \true\}} \Pr(x).
  \end{align}
  The $\group$-invariance of $\evidence$ allows us to evaluate $\evidence$ on only
  $x$ without considering other elements of $\orb(x)$.
  To compute $\Pr(\evidence)$, compute
  \begin{align}
    \sum_{\{x \in \Omega / \group ~\mid~ \evidence(x) = \true\}} |\orb(x)| \times \Pr(x).
  \end{align}
  Both of these can be accomplished in $\poly(|\Omega / \group|)$.
\end{proof}
\citet{NiepertAAAI14} identified a connection between bounded-width
\emph{exchangeable decompositions} and tractable (i.e., domain-lifted) exact
probabilistic inference using the above approach. Exchangeable decompositions are a particular kind of
$\group$-invariance. Let $\Pr(\varX_1, \varX_2, \cdots, \varX_n)$ be a
distribution on sets of variables $\varX_i$. Let $S_n$ be a group of all
permutations on a set of size $n$. Then, this distribution has an
exchangeable decomposition along $\{\varX_i\}$ if, for any $g \in S_n$:
\begin{align*}
  \Pr(\varX_1, \varX_2, \cdots, \varX_n) = \Pr(\varX_{g \cdot 1}, \varX_{g \cdot 2}, \cdots, \varX_{g \cdot n})
\end{align*}
\citet{NiepertAAAI14} showed how to perform exact lifted probabilistic inference
on any distribution with a fixed-width exchangeable decomposition by directly
constructing canonical representatives. However, this construction does not
generalize to other kinds of symmetries, and thus cannot be applied to factor
graphs which may have arbitrarily complex symmetric structure. In the next
section, we show how to apply Theorem~\ref{thm:tractability} to factor graphs.

\subsection{ORBIT GENERATION}
The previous section shows that inference can be efficient if we can (1)
construct representatives of each orbit class, (2) compute how large each orbit
is. In this section, we give an algorithm for performing these two operations
for colored factor graphs. First, we describe how to encode variable assignments
directly into the colored factor graph,
allowing us to leverage graph isomorphism
tools to compute canonical representatives and orbit sizes for
assignments to variables in factor graphs.
This colored assignment encoding is our
key technical contribution, and forms a foundation for our exact and approximate
inference algorithms. Then, we will
give a breadth-first search procedure for generating all canonical
representatives of a colored factor graph.

\subsubsection{Encoding Assignments}
\label{sec:encoding}
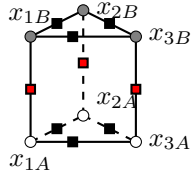
\begin{figure}
  \centering
   \begin{tikzpicture}[factor node/.style={rectangle,draw,thick,solid, scale=0.5},scale=0.7]

   \node[factor node, fill=red] at (0, 1.5) {};
   \node[factor node, fill=red] at (-1, 1) {};
   \node[factor node, fill=red] at (1, 1) {};

   \node[factor node, fill=black] at (-0.5, 0.25) {};
   \node[factor node, fill=black] at (0.5, 0.25) {};
   \node[factor node, fill=black] at (-0.2, 0) {};

   \node[factor node, fill=black] at (-0.5, 2.25) {};
   \node[factor node, fill=black] at (0.5, 2.25) {};
   \node[factor node, fill=black] at (-0.2, 2) {};

   \begin{scope}[on background layer]
     \draw[dashed, thick] (0, 0.5) -- (0, 2.5);
     \draw[thick] (1, 0) -- (1, 2);
     \draw[thick] (-1, 0) -- (-1, 2);

     \draw[dashed, thick] (-1,0) -- (0,0.5);
     \draw[dashed, thick] (0,0.5) -- (1,0);
     \draw[thick] (-1, 0) -- (1, 0);

     \draw[thick] (-1,2) -- (0,2.5);
     \draw[thick] (0,2.5) -- (1,2);
     \draw[thick] (-1, 2) -- (1, 2);
    
    \draw[fill=white] (-1, 0) circle (3pt) node[below=0.5mm] {$x_{1A}$};
    \draw[fill=white] (0, 0.5) circle (3pt)  node[right=0.5mm,yshift=0.2cm] {$x_{2A}$};
    \draw[fill=gray] (1, 2) circle (3pt) node[right=0.5mm] {$x_{3B}$};
    \draw[fill=white] (1, 0) circle (3pt) node[right=0.5mm] {$x_{3A}$};
    \draw[fill=gray] (0, 2.5) circle (3pt) node[right=0.5mm] {$x_{2B}$};
    \draw[fill=gray] (-1, 2) circle (3pt) node[above=0.5mm] {$x_{1B}$};
   \end{scope}
\end{tikzpicture}
\caption{A colored graph of the 3-pigeon 2-hole problem that encodes the
  assignment $\varx=000~111$. True variable nodes are gray and false variable
  nodes are white.}
\label{fig:assgnencoding}
\end{figure}
Our objective in this section is to leverage graph isomorphism tools to compute
the key quantities necessary for applying the procedure described in
Theorem~\ref{thm:tractability} to factor graphs.
Let $\graph$ be the induced colored graph of
$\fgraph$.
As terminology, an element $\varx \in \bool^\varX$ is an
assignment to variables $\varX$.
We will use graph isomorphism tools to construct (1) a canonization
function for variable assignments, $\sigma : \bool^\varX \rightarrow \bool^\varX / \aut(\graph)$; and (2)
the size of the orbit of $\varx \in \bool^\varX$ under $\aut(\graph)$.
To do this, we encode assignments directly into the colored
factor graph, which to our knowledge is a novel construction in
this context:

\begin{definition}
  Let $\fgraph = (\varX, F)$ be a factor graph, let $\varx \in \bool^\varX$,
and let $\graph = (V,E,C)$ be the colored graph induced by $\fgraph$. Then the
\emph{assignment-encoded colored graph}, denoted $\graph(\fgraph, \varx)$, is
the colored graph that colors the variable nodes that are true and false in
$\varx$ with distinct colors in $\graph$.
\end{definition}

An example is shown in Figure~\ref{fig:assgnencoding}, which shows an encoding of
the assignment $000~111$. The assignment
$000~111$ is isomorphic to the assignment $111~000$ under the action of
$\aut(\graph)$, specifically flipping holes. Then, assignments that are in the same
orbit under $\aut(\graph)$ have isomorphic colored graph encodings:
\begin{theorem}
  Let $\fgraph = (\varX, F)$ be a factor graph, $\graph$ be its colored graph
encoding, and $\varx, \varx' \in \bool^\varX$. Then, $\varx \sim \varx'$
under the action of $\aut(\graph)$ iff $\graph(\fgraph, \varx) \cong
\graph(\fgraph, \varx')$.
\label{thm:coloredcorrespondence}
\vspace{-0.8cm}
\end{theorem}
\begin{proof}
  See Appendix~\ref{sec:coloredcorrespondenceproof}.
\end{proof}
\paragraph{Canonization} Our goal now is to use graph
isomorphism tools to construct a canonization function for variable assignments.
In particular, it maps all isomorphic assignments to exactly one member of their
orbit. We will rely on \emph{colored graph canonization}, a well-studied problem
in graph theory for which there exist many implementations \citep{Mckay2014}:

\begin{definition}
  Let $\graph = (V,E,C)$ be a colored graph. Then a \emph{colored graph
canonization} is a canonization function $\sigma: V \rightarrow V /
\aut(\graph)$.
\end{definition}
A colored graph canonization function applied to Figure~\ref{fig:assgnencoding} will
select exactly one color-isomorphic vertex configuration as the canonical one,
for example putting all pigeons in hole $A$.
Then, the canonization of the assignment-encoded colored graph is a canonization
of variable assignments:
\begin{definition}
  Let $\fgraph = (\varX, F)$ and $\varx = \{(x, v)\}$ be a variable assignment,
where $x \in \varX$ and $v \in \bool$. Let $\sigma_{\graph(\fgraph, \varx)}$ be
a canonization of $\graph(\fgraph, \varx)$. Then, let $\sigma': \bool^\varX
\rightarrow \bool^\varX$ be defined $\sigma'(\varx) = \{
(\sigma_{\graph(\fgraph, \varx)}(x), v) \mid (x,v) \in \varx \}$. Then
$\sigma'$ is called the \emph{induced variable canonization} of $\bool^\varX$.
\end{definition}
Intuitively, an induced variable canonization computes the canonization
of the assignment-encoded colored graph, and then applies that canonization
function to variables. Then,
\begin{proposition}
For a factor graph $\fgraph$ with colored graph $\graph$, the induced variable
canonization is a canonization function $\bool^\varX \rightarrow \bool^\varX /
\aut(\graph)$.
\end{proposition}

\paragraph{Computing the size of an orbit} 
Theorem~\ref{thm:tractability} requires efficiently computing
the size of the orbit of an assignment. To accomplish this,
we will apply the
orbit-stabilizer theorem in a manner similar to \citet{Niepert2013}.
The size of a stabilizer is related to the size of an orbit with the following
well-known theorem:
\begin{theorem}[Orbit-stabilizer]
  Let $\group$ act on $\Omega$. Then for any $x \in \Omega$, $|\group| =
  |\stab(x)| \times |\orb(x)|$.
\end{theorem}
Thus, to compute orbit size of assignments $\varx$, we will compute (1) the
order of the $\stab(\varx)$ under $\aut(\graph)$; and (2) the order of
$\aut(\graph)$. Now we can again use graph isomorphism tools. The
stabilizer of $\varx$ corresponds with the automorphism group of the
colored graph encoding of $\varx$. To see this, observe that a permutation that
relabels pigeons but leaves holes fixed is a stabilizer of the assignment in
Figure~\ref{fig:assgnencoding}; this permutation is also a member of the
color-automorphism group of the graph. Formally:
\begin{theorem}
  Let $\fgraph = (\varX, F)$ be a factor graph with colored graph encoding
$\graph$. Then for any $\varx \in \bool^\varX$, $\stab_{\aut(\graph)}(\varx) =
\aut(\graph(\fgraph, \varx))$.
\label{thm:stabilizercorrespondence}
\end{theorem}
  The proof can be found in Appendix~\ref{sec:stabilizercorrespondence}.
Thus we have reduced computing orbit sizes to computing group orders, which can
be computed efficiently using computational group
theory tools such as $\mathtt{GAP}$ \citep{GAP4, Seress2003}.
Thus if we can
exhaustively generate canonical representatives, then we can perform lifted exact
inference. The next section shows how to do this.

\subsubsection{Generating All Canonical Representatives}
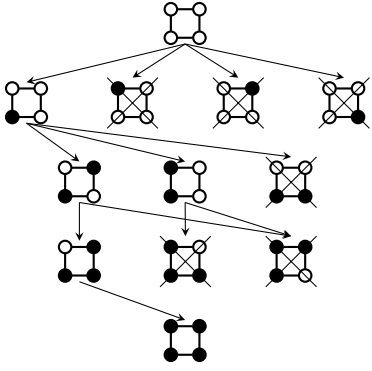
\begin{figure}
  \centering
  \begin{tikzpicture}[factor node/.style={circle,draw,thick,solid, scale=0.5},]
    
    \begin{scope}[shift={(0em,0em)}, node distance=0.2cm, local bounding box=A0]
      \node[factor node] (A) {};
      \node[factor node, below=of A] (B) {};
      \node[factor node, right=of A] (C) {};
      \node[factor node, right=of B] (D) {};
      \draw[-, thick] (A) -- (B);
      \draw[-, thick] (A) -- (C);
      \draw[-, thick] (B) -- (D);
      \draw[-, thick] (C) -- (D);
    \end{scope}

    \begin{scope}[shift={(-6em,-3em)}, node distance=0.2cm, local bounding box = B0]
      \node[factor node] (A) {};
      \node[factor node, below=of A, fill=black] (B) {};
      \node[factor node, right=of A ] (C) {};
      \node[factor node, right=of B] (D) {};
      \draw[-, thick] (A) -- (B);
      \draw[-, thick] (A) -- (C);
      \draw[-, thick] (B) -- (D);
      \draw[-, thick] (C) -- (D);
    \end{scope}

    \begin{scope}[shift={(-2em,-3em)}, node distance=0.2cm, local bounding box = B1]
      \draw (0.55em,-0.55em) node[cross=1em] {};
      \node[factor node, fill=black] (A) {};
      \node[factor node, below=of A] (B) {};
      \node[factor node, right=of A] (C) {};
      \node[factor node, right=of B] (D) {};
      \draw[-, thick] (A) -- (B);
      \draw[-, thick] (A) -- (C);
      \draw[-, thick] (B) -- (D);
      \draw[-, thick] (C) -- (D);
    \end{scope}

    \begin{scope}[shift={(2em,-3em)}, node distance=0.2cm, local bounding box = B2]
      \draw (0.55em,-0.55em) node[cross=1em] {};

      \node[factor node] (A) {};
      \node[factor node, below=of A] (B) {};
      \node[factor node, right=of A, fill=black] (C) {};
      \node[factor node, right=of B] (D) {};
      \draw[-, thick] (A) -- (B);
      \draw[-, thick] (A) -- (C);
      \draw[-, thick] (B) -- (D);
      \draw[-, thick] (C) -- (D);
    \end{scope}
   
    \begin{scope}[shift={(6em,-3em)}, node distance=0.2cm, local bounding box = B3]
      \draw (0.55em,-0.55em) node[cross=1em] {};
      \node[factor node] (A) {};
      \node[factor node, below=of A] (B) {};
      \node[factor node, right=of A] (C) {};
      \node[factor node, right=of B, fill=black] (D) {};
      \draw[-, thick] (A) -- (B);
      \draw[-, thick] (A) -- (C);
      \draw[-, thick] (B) -- (D);
      \draw[-, thick] (C) -- (D);
    \end{scope}

    \begin{scope}[shift={(-4em,-6em)}, node distance=0.2cm, local bounding box = C0]
      \node[factor node] (A) {};
      \node[factor node, below=of A, fill=black] (B) {};
      \node[factor node, right=of A, fill=black ] (C) {};
      \node[factor node, right=of B] (D) {};
      \draw[-, thick] (A) -- (B);
      \draw[-, thick] (A) -- (C);
      \draw[-, thick] (B) -- (D);
      \draw[-, thick] (C) -- (D);
    \end{scope}

    \begin{scope}[shift={(0em,-6em)}, node distance=0.2cm, local bounding box = C1]
      \node[factor node, fill=black] (A) {};
      \node[factor node, below=of A, fill=black] (B) {};
      \node[factor node, right=of A] (C) {};
      \node[factor node, right=of B] (D) {};
      \draw[-, thick] (A) -- (B);
      \draw[-, thick] (A) -- (C);
      \draw[-, thick] (B) -- (D);
      \draw[-, thick] (C) -- (D);
    \end{scope}

    \begin{scope}[shift={(4em,-6em)}, node distance=0.2cm, local bounding box = C2]
      \draw (0.55em,-0.55em) node[cross=1em] {};

      \node[factor node] (A) {};
      \node[factor node, below=of A, fill=black] (B) {};
      \node[factor node, right=of A] (C) {};
      \node[factor node, right=of B, fill=black] (D) {};
      \draw[-, thick] (A) -- (B);
      \draw[-, thick] (A) -- (C);
      \draw[-, thick] (B) -- (D);
      \draw[-, thick] (C) -- (D);
    \end{scope}

     \begin{scope}[shift={(-4em,-9em)}, node distance=0.2cm, local bounding box = D0]

      \node[factor node] (A) {};
      \node[factor node, below=of A, fill=black] (B) {};
      \node[factor node, right=of A, fill=black] (C) {};
      \node[factor node, right=of B, fill=black] (D) {};
      \draw[-, thick] (A) -- (B);
      \draw[-, thick] (A) -- (C);
      \draw[-, thick] (B) -- (D);
      \draw[-, thick] (C) -- (D);
    \end{scope}

    \begin{scope}[shift={(0em,-9em)}, node distance=0.2cm, local bounding box = D1]
      \draw (0.55em,-0.55em) node[cross=1em] {};
      \node[factor node, fill=black] (A) {};
      \node[factor node, below=of A, fill=black] (B) {};
      \node[factor node, right=of A] (C) {};
      \node[factor node, right=of B, fill=black] (D) {};
      \draw[-, thick] (A) -- (B);
      \draw[-, thick] (A) -- (C);
      \draw[-, thick] (B) -- (D);
      \draw[-, thick] (C) -- (D);
    \end{scope}

    \begin{scope}[shift={(4em,-9em)}, node distance=0.2cm, local bounding box = D2]
      \draw (0.55em,-0.55em) node[cross=1em] {};
      \node[factor node, fill=black] (A) {};
      \node[factor node, below=of A, fill=black] (B) {};
      \node[factor node, right=of A, fill=black] (C) {};
      \node[factor node, right=of B] (D) {};
      \draw[-, thick] (A) -- (B);
      \draw[-, thick] (A) -- (C);
      \draw[-, thick] (B) -- (D);
      \draw[-, thick] (C) -- (D);
    \end{scope}

     \begin{scope}[shift={(0em,-12em)}, node distance=0.2cm, local bounding box = E0]
      \node[factor node, fill=black] (A) {};
      \node[factor node, below=of A, fill=black] (B) {};
      \node[factor node, right=of A, fill=black] (C) {};
      \node[factor node, right=of B, fill=black] (D) {};
      \draw[-, thick] (A) -- (B);
      \draw[-, thick] (A) -- (C);
      \draw[-, thick] (B) -- (D);
      \draw[-, thick] (C) -- (D);
    \end{scope}

    \draw[->,>=stealth] (A0.south) -- (B0.north);
    \draw[->,>=stealth] (A0.south) -- (B1.north);
    \draw[->,>=stealth] (A0.south) -- (B2.north);
    \draw[->,>=stealth] (A0.south) -- (B3.north);
    \draw[->,>=stealth] (B0.south) -- (C0.north);
    \draw[->,>=stealth] (B0.south) -- (C1.north);
    \draw[->,>=stealth] (B0.south) -- (C2.north);
    \draw[->,>=stealth] (C0.south) -- (D0.north);
    \draw[->,>=stealth] (C0.south) -- (D2.north);
    \draw[->,>=stealth] (C1.south) -- (D1.north);
    \draw[->,>=stealth] (C1.south) -- (D2.north);
    \draw[->,>=stealth] (D0.south) -- (E0.north);
  \end{tikzpicture}
  \caption{Example breadth-first search tree, read top-down.
White
nodes encode false assignments, and black nodes encode true assignments.
  }
  \label{fig:bfs}
\end{figure}

Our algorithm for generating canonical representatives is a simple breadth-first
search that relies on assignment canonization. This procedure is a kind of
\emph{isomorph-free exhaustive generation}, and there exist more sophisticated
procedures than the one we present here \citep{McKay1998}.

Let $\varx$ be some variable assignment. Then, an \emph{augmentation} of
$\varx$ is a copy of $\varx$ with one variable that was previously false
assigned to true. We denote the set of all augmentations as
$\mathcal{A}(\varx)$. Our breadth-first search tree will be defined by a series of
augmentations as follows:

\begin{enumerate}[noitemsep,leftmargin=*]
\item Nodes of the search tree are assignments $\varx$.
\item The root of the tree is the all false assignment.
\item Each level $L$ of the search tree has exactly $L$ true assignments to variables.
\item Nodes are expanded until level $|\varX|$.
\item Before expanding a node, check if it is not isomorphic to one that has
  already been expanded by computing its canonical form.
\item Then, expand a node $\varx$ by adding $\mathcal{A}(\varx)$ to the
  frontier.
\end{enumerate}

An example of this breadth-first search procedure is visualized in
Figure~\ref{fig:bfs}. The search is performed on a 4-variable factor graph that has
one factor on each edge, and all factors are symmetric. The factors are elided
in the figure for visual clarity. Each arrow represents an augmentation.  Crossed
out graphs are pruned due to being isomorphic with a previously expanded node.

Now we bound the number of required graph isomorphism calls for this search
procedure:
\begin{theorem}
  For a factor graph $\fgraph = (\varX, F)$ with $|\bool^\varX/\aut(\graph)|$
canonical representatives, the above breadth-first search requires at most
$|\varX| \times |\bool^\varX/\aut(\graph)|$ calls to a graph isomorphism tool.
\vspace{-0.2cm}
\end{theorem}
\begin{proof}
  There are at most $|\bool^\varX / \aut(\graph)|$ expansions, and each
expansion adds at most $|\varX|$ nodes to the frontier. A canonical form must be
computed for each node that is added to the frontier.
\end{proof}

\textbf{Pruning expansions}\quad This expansion process can be further optimized by
preemptively reducing the number of nodes that are added to the frontier in
Step 6, using the following lemma:
\begin{lemma}[Expansion Pruning]
  Let $\fgraph$ be a factor graph, $\varx$ be a variable assignment, and
$\varx_1, \varx_2$ be augmentations of $\varx$ that update variables $x$ and
$y$ respectively. Then, $\varx_1 \sim \varx_2$ under $\aut(\graph)$ if $x$
and $y$ are in the same variable orbit under $\aut(\graph(\fgraph, \varx))$.
\label{lem:pruning}
\end{lemma}
The proof is in Appendix~\ref{sec:pruningproof}. Using this lemma we can update
Step 6 to only include a single element of each variable orbit of $\varX$ under
$\aut(\graph(\fgraph, \varx))$.

\subsection{EXACT LIFTED INFERENCE ALGORITHM}

Now we combine the theory of the previous two sections to perform exact lifted
inference on factor graphs. Algorithm~\ref{alg:liftedinf} performs exact lifted
inference via a breadth-first search over canonical assignments. Variable $r$
holds a set of canonical representatives, $q$ holds the frontier, $p$
accumulates the unnormalized probability of the evidence, and $Z$ accumulates
the normalizing constant. A graph isomorphism tool is used to compute $\sigma$
on Line~5.
Each time the algorithm finds a new representative, it
computes the size of the orbit using the orbit stabilizer theorem on Line 9;
$\mathtt{GAP}$ is used to compute the order of these permutation groups.
Lemma~\ref{lem:pruning} is used on Line 13 to avoid adding augmentations to the frontier
that are known a-priori to be isomorphic to prior ones. This algorithm can be
easily modified to produce the MPE by simply returning the canonical
representative from $r$ with the highest probability.

\textbf{Experimental Evaluation}\quad
To validate our method we implemented Algorithm~\ref{alg:liftedinf} using the
$\mathtt{Sage}$ math library, which wraps $\mathtt{GAP}$ and a graph isomorphism
tool~\citep{sagemath}.\footnote{The source code for our
exact and approximate inference algorithms can be found at
\url{https://github.com/SHoltzen/orbitgen}.} We compared our lifted inference
procedure against $\mathtt{Ace}$, an exact inference tool for discrete Bayesian
networks that is unaware of the symmetry of the model~\citep{ChaviraDarwiche2005}.
Figure~\ref{fig:exactexp} shows experimental results for performing exact lifted
inference on two families of factor graphs.
The first is a class of pairwise factor graphs that have an identical symmetric potential
between all nodes, with one factor (in red) designated as an evidence factor:
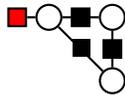
\begin{center}
   \begin{tikzpicture}[->,>=stealth,auto,node distance=1.2em,
      thick,main node/.style={circle,draw},factor node/.style={rectangle,
        fill=black, draw}]
      \node[main node] (A) {};
      \node[factor node] (F1) [right of=A] {};
      \node[main node] (B) [right of=F1] {};
      \node[factor node] (F2) [below of=B] {};
      \node[main node] (C) [below of=F2] {};

      \node[factor node] (E) [left of=A, fill=red] {};
      \draw [-] (A) to (B);
      \draw [-] (B) to (C);
      \draw [-] (C) -- node[fill=black,yshift=0.4em,xshift=0.4em] {} (A);
      \draw [-] (A) -- (E);
    \end{tikzpicture}
  \end{center}
We also evaluated our method on the pigeonhole problem from Section~\ref{sec:motivation} with
two holes and increasing number of pigeons. In both experiments, the number of
orbits grows linearly, even though there is little independence. Thus,
$\mathtt{Ace}$ scales exponentially, since the treewidth grows quickly, while
our method scales sub-exponentially. To our knowledge, this is the first
example of performing exact inference on this family of models.
\begin{algorithm}[t]
  \KwData{A factor graph $\fgraph = (\varX, F)$ with color encoding
    $\graph$; $\aut(\graph)$-invariant evidence $\evidence{}$}
  \KwResult{The probability of evidence $\Pr(\evidence)$}
  $r \leftarrow$ empty set, $p \leftarrow 0$, $Z \leftarrow 0$\;
  $q \leftarrow$ queue containing the all-false assignment\;
  \While{$q$ is not empty}{
    $\varx \leftarrow q$.pop()\;
    $\mathtt{Canon} \leftarrow \sigma(\graph(\fgraph, \varx))$
    \tcp*{Invoke graph iso. tool}
    \If{$\mathtt{Canon} \in r$}{
      continue\;
    }
    Insert $\mathtt{Canon}$ into $r$\;
    $|\orb(\varx)| \leftarrow |\aut(\graph)| / |\aut(\graph(\fgraph, \varx))|$
    \tcp*{Invoke $\mathtt{GAP}$}
    \If{$\evidence(\varx) = \true$}{
      $p \leftarrow p + |\orb(\varx)| \times F(\varx)$\;
    }
    $Z \leftarrow Z + |\orb(\varx)| \times F(\varx)$\;
    \For{$o$ from each variable orbit of $\aut(\graph(\fgraph, \varx))$}{
      \If{$o$ is a false variable}{
        $\varx' \leftarrow \varx$ with $o$ true\;
        Append $\varx'$ to $q$\;
      }
    }
  }
  \Return $p / Z$
  \caption{$\mathtt{ExactLiftedInference}(\fgraph, \evidence)$}
  \label{alg:liftedinf}
\end{algorithm}

\begin{figure}
\begin{subfigure}[b]{0.5\linewidth}
  \centering
  \begin{tikzpicture}
	\begin{axis}[
		height=3cm,
		width=4cm,
		grid=major,
    xlabel={\# Variables},
    no markers,
    ylabel={Time (s)},
	]

  \addplot[densely dashed, very thick] coordinates {
(10,	0.335)
(15,	0.559)
(20,	1.115)
(25,	2.426)
(30,	5.298)
(35,	9.998)
(40,	20.81)
(45,	36.179)
(50,	70.331)
	};
  
  \addplot[very thick] coordinates {
    (10, 0.06)
    (15, 0.05)
    (20, 1.61)
    (22, 6.84)
    (23, 13.35)
    (24, 182.55)
	} node[draw, fill=red, circle, scale=0.4] {};
  
	\end{axis}
\end{tikzpicture}
\caption{Inference for pairwise factor graph.}
\end{subfigure}~
  \begin{subfigure}[b]{0.5\linewidth}
  \centering
  \begin{tikzpicture}
	\begin{axis}[
		height=3cm,
		width=4cm,
		grid=major,
    no markers,
    xlabel={\# Pigeons},
    legend columns=2,
    legend style={at={(3em,6em)},anchor=north}
	]

  \addplot[densely dashed, very thick] coordinates {
    (2,	0.052)
    (4,	0.096)
    (8,	0.342)
    (10,	0.608)
    (12,	1.01)
    (14,	1.6298)
    (16,	2.572)
    (18,	3.898)
    (20,	5.914)
    (22,	7.1)
    (24,	9.53)
    (30, 22)
    (32, 29)
    (34, 39)
    (35, 47)
    (36, 52)
    (37, 59)
    (40, 83)
    (45, 146)
    (50, 243)
  }; 
	\addlegendentry{Orbit Gen.}
  
  \addplot[very thick] coordinates {
    (3,	0.06)
    (8,	0.09)
    (12,	0.1)
    (14,	0.25)
    (16,	0.46)
    (18,	1.35)
    (20, 663)
	} node[draw, fill=red, circle, scale=0.4] {};
	\addlegendentry{$\mathtt{Ace}$}
  
	\end{axis}
\end{tikzpicture}
\caption{Inference for 2-hole pigeonhole problem.}
\end{subfigure}

\caption{Evaluation of Algorithm~\ref{alg:liftedinf}. A red circle indicates that
  $\mathtt{Ace}$ ran out of memory at that time.}
\label{fig:exactexp}
\end{figure}
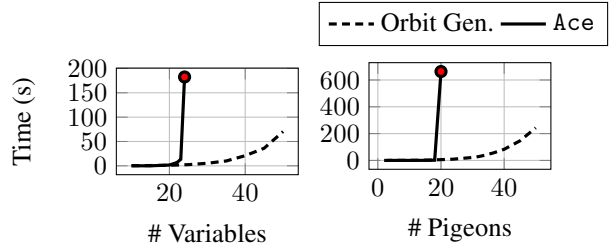

\section{ORBIT-JUMP MCMC}
\label{sec:orbsamp}

In this section we introduce
\emph{orbit-jump MCMC}, an MCMC algorithm that mixes quickly when the
distribution has few orbits, at the cost of requiring multiple graph isomorphism
calls for each transition. The algorithm is summarized in Algorithm~\ref{alg:orbitjump}.
Orbit-jump MCMC is an alternative to
Lifted MCMC \citep{Niepert2012, Niepert2013} 
that
generates provably high-quality samples at the expense of more costly
transitions.
Lifted MCMC exploits symmetric structure to
quickly transition \emph{within} orbits. Lifted MCMC is
efficient to implement: it requires only a single call to a graph isomorphism
tool. However, lifted MCMC relies on Gibbs sampling to jump \emph{between
orbits}, and therefore has no guarantees about its mixing time
for distributions with few orbits.  
Orbit-jump MCMC is a Metropolis-Hastings MCMC algorithm that uses the following
distribution as its proposal:

\begin{definition}
Let $\group$ act on $\Omega$. Then for $x \in \Omega$, the \emph{uniform orbit
distribution} is:
\begin{align}
  \Pr_{\Omega / \group}(x) \triangleq \frac{1}{|\Omega/\group| \times |\orb(x)|}
\end{align}
This is the probability of uniformly choosing an
orbit $o \in \Omega / \group$, and then sampling uniformly from $\sigma^{-1}(o)$.
\end{definition}

The \emph{orbit-jump MCMC chain} for a $\group$-invariant distribution $\Pr$
is defined as follows, initialized to $x \in \Omega$:
\begin{enumerate}[noitemsep, leftmargin=*]
\item Sample $x' \sim \Pr_{\Omega / \group}$;
\item Accept $x'$ with probability $
   \min\left(1, \frac{\Pr(x') \times
    |\orb(x')|}{\Pr(x) \times |\orb(x)|}\right)$
\end{enumerate}
This Markov chain has $\Pr$ as its stationary
distribution.
Orbit-jump MCMC has a high probability of proposing transitions \emph{between}
orbits, which is an alternative to the
within-orbit exploration of lifted MCMC.\footnote{ This proposal is independent
  of the previous state, a scheme that is
sometimes called \emph{Metropolized independent
sampling} (MIS)~\citep{liu1996metropolized}. Importance sampling is an
alternative to MIS. We use MIS rather than importance sampling in order to make
the connection with lifted MCMC more explicit. }

Next we will describe how to sample from $\Pr_{\Omega / \group}$ using an MCMC
method known as the \emph{Burnside process}.
Then, we will discuss the mixing time of this proposal, and prove that it
mixes in the number of orbits of the distribution.

\begin{figure}
  \centering
  \begin{tikzpicture}[scale=0.7]
    \node at(-12em,0em) {$\group$};
    \node at(-12em,-4em) {$\Omega$};

    \node at(-3em, 0em) (ident) {$(A)(B)$};
    \node at(3em, 0em) (cyc) {$(A~B)$};

    \begin{scope}[shift={(-9em,-4em)}, local bounding box=A0]
      \node[draw, circle, scale=0.5] (A) {};
      \node[draw, circle, right=0.2cm of A, scale=0.5] (B) {};
      \node[below=0.1cm of A] {$A$};
      \node[below=0.1cm of B] {$B$};
      \draw[-, thick] (A) -- (B);
    \end{scope}

    \begin{scope}[shift={(-3em,-4em)}, local bounding box = A1 ]
      \node[draw, circle, fill=black, scale=0.5] (A) {};
      \node[draw, circle, right=0.2cm of A, scale=0.5] (B) {};
      \draw[-, thick] (A) -- (B);
    \end{scope}

    \begin{scope}[shift={(3em,-4em)},local bounding box = A2 ]
      \node[draw, circle, scale=0.5] (A) {};
      \node[draw, circle, right=0.2cm of A, , fill=black, scale=0.5] (B) {};
      \draw[-, thick] (A) -- (B);
    \end{scope}
    \begin{scope}[shift={(9em,-4em)},local bounding box = A3 ]
      \node[draw, circle, fill=black, scale=0.5] (A) {};
      \node[draw, circle, right=0.2cm of A, fill=black, scale=0.5] (B) {};
      \draw[-, thick] (A) -- (B);
    \end{scope}

    \draw[-] (A0.north) -- (ident.south);
    \draw[-] (A1.north) -- (ident.south);
    \draw[-] (A2.north) -- (ident.south);
    \draw[-] (A3.north) -- (ident.south);

    \draw[-] (A0.north) -- (cyc.south);
    \draw[-] (A3.north) -- (cyc.south);
    
  \end{tikzpicture}
  \caption{Illustration of the Burnside process on a colored graph with two
    nodes and two colors.
  }
  \label{fig:burnside}
\end{figure}
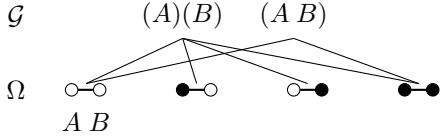

\subsection{SAMPLING FROM $\Pr_{\Omega /\group}$}
\citet{Jerrum1992} gave an MCMC technique known as the \emph{Burnside process}
for drawing samples from $\Pr_{\Omega / \group}$. The Burnside process is a
Markov Chain Monte Carlo method defined as follows, beginning from some $x \in
\Omega$:
\begin{enumerate}[noitemsep, leftmargin=*]
\item Sample $g \sim \stab(x)$ uniformly;
\item Sample $x \sim \fix(g)$ uniformly, where $\fix(g) = \{x \in \Omega \mid g
\cdot x = x\}$. We call elements of $\fix(g)$ \emph{fixers}.
\end{enumerate}

\begin{theorem}[\citet{Jerrum1992}] The stationary distribution of the Burnside
process is equal to $\Pr_{\Omega / \group}$.
\end{theorem}

This process can be visualized as a random walk on a bipartite graph. One set of
nodes are elements of $\Omega$, and the other set are elements of $\group$. There
is an edge between $x \in \Omega$ and $g \in \group$ iff $g \cdot x = x$.

An example of this bipartite graph is shown in Figure~\ref{fig:burnside}. The
set $\Omega$ is the set of 2-node colored graphs, and the group $\group = S_2$ permutes the
vertices of the graph. The identity element $(A)(B)$ stabilizes all elements of
$\Omega$, and so has an edge to every element in $x$; $(A~B)$
only stabilizes graphs whose vertices have the same color.

\citet{Jerrum1992} proved that the Burnside process mixes rapidly for several
important groups, but it does not always mix quickly \citep{Goldberg2002}. In
such cases, it is important to draw sufficient samples from the Burnside process
in order to guarantee that the orbit-jump proposal is unbiased. Next we will
describe how to implement the Burnside process on factor graphs using the
machinery from Section~\ref{sec:encoding}.



\subsubsection{Burnside Process on Factor Graphs}
For $\group$ acting on a set of variables $\varX$, the Burnside process requires
the ability to (1) draw samples uniformly from the stabilizer subgroup of an
assignment to variables, and (2) sample a random fixer for any group element in
$\group$. Here we describe how to perform these two computations for a colored
factor graph $\fgraph = (\varX, F)$.\footnote{This process is conceptually similar to the procedure
for randomly sampling orbits in the P\'olya-theory setting described by
\citet{Goldberg2001}, but this is the first time that this procedure is applied
directly to factor graphs} This procedure is summarized in lines 3--7 in
Algorithm~\ref{alg:orbitjump}. 

\textbf{Stabilizer Sampling}\quad
Section~\ref{sec:encoding} showed how to compute the stabilizer group of $\varx \in
\bool^\varX$ using graph isomorphism tools. To sample uniformly from this
stabilizer group,
we rely on the \emph{product replacement algorithm}, which is an efficient
procedure for uniformly sampling group elements \citep{Pak00}.

\textbf{Fixer Sampling}\quad
Let $g \in \group$ be a permutation that acts on the vertices of a colored
factor graph. Then we uniformly sample an assignment-encoded colored factor
graph that is fixed by $g$ in the following way. First, decompose $g$ into a
product of disjoint cycles. Then, for each cycle that
contains variable nodes, choose a truth assignment uniformly randomly, and then
color the vertices in that
cycle with that color. This colored graph is fixed by $g$
and is uniformly random by the independence of coloring each cycle and the fact
that all colorings fixed by $g$ can be obtained in this manner.


\begin{algorithm}[t]
  \KwData{A factor graph $\fgraph = (\varX, F)$, a point $\varx \in
    \bool^\varX$, number of Burnside process steps $k$}
  \caption{A step of Orbit-jump MCMC}
  \label{alg:orbitjump}
  $\varx' \leftarrow \varx$\;
  \For{$i \in \{1, 2, \cdots, k\}$}{
    $\group_{\stab} \leftarrow \aut(\graph(\fgraph, \varx'))$
     \tcp*{Invoke graph iso. tool}
  Sample $s \sim \group_{\stab}$ using product replacement\;
  \For{Each variable cycle $c$ of $s$}{
    $v \sim $ Bernoulli($1/2$)\;
    Assign all variables $c$ in $\varx'$ to $v$\;
  }
  }
  Accept $\varx'$ with probability $\min\Big(1, \frac{F(\varx') \times
    |\orb(\varx')|}
  {F(\varx) \times |\orb(\varx)|}\Big)$
\end{algorithm}

\subsection{MIXING TIME OF ORBIT-JUMP MCMC}
The \emph{total variation distance} between two discrete probability measures
$\mu$ and $\nu$ on $\Omega$, denoted $d_{TV}(\mu, \nu)$, is:
\begin{align}
  \label{eq:tv}
  d_{TV}(\mu, \nu) = \frac{1}{2}\sum_{x \in \Omega} |\mu(x) - \nu(x)|.
\end{align}
The \emph{mixing time} of a Markov chain is the minimum number of iterations
that the chain must be run starting in any state until the total variation
distance between the chain and its stationary distribution is less than some
parameter $\varepsilon > 0$. The mixing time of orbit-jump MCMC can be bounded in
terms of the number of orbits, which is a property not enjoyed by lifted MCMC:
\begin{theorem}
  Let $\Pr$ be a $\group$-invariant distribution on $\Omega$ and let $P$ be the
  transition matrix of orbit-jump MCMC. Then, for any $x \in \Omega$,
  $d_{TV}(P^tx, \Pr) \le \left( \frac{|\Omega / \group|-1}{|\Omega /\group|} \right)^t$.
  It follows that for any
$\varepsilon > 0$,  $d_{TV}(P^tx, \Pr) \le \varepsilon$ if $t \ge
\log(\varepsilon^{-1}) \times |\Omega / \group|$.
  \label{thm:mixing}
\end{theorem}
  For a detailed proof see Appendix~\ref{sec:mixing}. Note that the bound on
this mixing time does not take into account the cost of drawing samples from
$\Pr_{\Omega / \group}$, which involves multiple graph isomorphism calls.

\textbf{Pigeonhole case study}\quad
We implemented the orbit-jump MCMC procedure on factor graphs using
$\mathtt{Sage}$. In order to evaluate the performance of orbit-jump MCMC, we
will compare the total variation distance of various MCMC procedures.
We experimentally compare the mixing time of
lifted MCMC~\citep{Niepert2012, Niepert2013} and our orbit-jump MCMC in
Figure~\ref{fig:mixing}, which computes the total variation distance of these
two MCMC methods from their stationary distribution as a function of the number
of iterations on two versions of the pigeonhole problem.\footnote{In these
  experiments, for each step of orbit-jump MCMC, we use $7$ steps of the
  Burnside process. } The first version in
Figure~\ref{fig:samplepigeon} is the motivating example with hard constraints
from Section~\ref{sec:motivation}. The second version in
Figure~\ref{fig:samplequantum} shows a \emph{``quantum'' pigeonhole problem},
where the constraint in Equation~\ref{eq:pigeonfactor} is relaxed so that
pigeons are allowed to be placed into multiple holes.

\begin{figure}
  \centering

  \begin{subfigure}{0.5\linewidth}
  \vspace{0.73cm}
    \centering
   \begin{tikzpicture}
	\begin{axis}[
		height=3.5cm,
		width=4cm,
		grid=major,
    ymax=1,
    no markers,
    every axis plot/.append style={very thick},
    xlabel={\# Iterations},
    ylabel={$d_{TV}$},
	]

  \addplot+ [
  densely dotted,
  red
  ] table [
  x=n,
  y=tv,
  ] {data/5pigeon-gibbs.txt};

  \addplot+ [
  blue,
  dash dot,
  ] table [
  x=n,
  y=tv,
  ] {data/upper-bound.txt};

  \addplot+ [
  black
  ] table [
  x=n,
  y=tv,
  ] {data/5pigeon-jump.txt};

	\end{axis}
\end{tikzpicture}
\caption{Hard pigeonhole.}
\label{fig:samplepigeon}
\end{subfigure}~
\begin{subfigure}{0.5\linewidth}
  \centering
  \begin{tikzpicture}
	\begin{axis}[
		height=3.5cm,
		width=4cm,
		grid=major,
    ymax=1,
    no markers,
    every axis plot/.append style={very thick},
    ymajorticks = false,
    xlabel={\# Iterations},
    legend style={at={(5.2em,7.7em)},anchor=north}
	]

  \addplot+ [
  dash dot,
  blue,
  ] table [
  x=n,
  y=tv,
  ] {data/upper-bound.txt};
  \addlegendentry{UB}

    \addplot+ [
  densely dotted,
  red
  ] table [
  x=n,
  y=tv,
  ] {data/quantum-gibbs.txt};
  \addlegendentry{Lifted}

  \addplot+ [
  black,
  ] table [
  x=n,
  y=tv,
  ] {data/quantum-jump.txt};
  \addlegendentry{Orbit-Jump}

\end{axis}
\end{tikzpicture}
\caption{Quantum pigeonhole.}
\label{fig:samplequantum}
\end{subfigure}
  \caption{Total variation distance between Markov chains and their stationary
distributions for a pigeonhole problem with 5 pigeons and 2 holes. ``Lifted'' is
lifted MCMC~\citep{Niepert2012} and ``UB'' is the upper bound predicted by
Theorem~\ref{thm:mixing}.}
  \label{fig:mixing}
\end{figure}
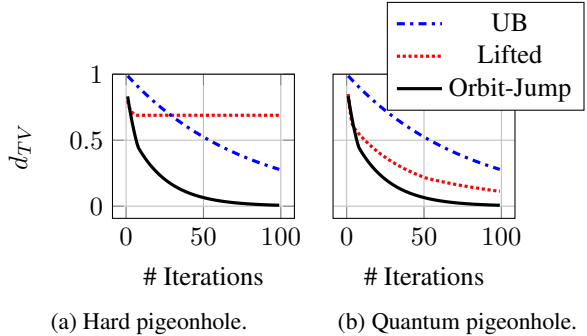

Lifted MCMC fails to converge in Figure~\ref{fig:samplepigeon} because it cannot
transition due to the hard constraint from Equation~\ref{eq:pigeonfactor}; this
illustrates that lifted MCMC can fail even for distributions with few orbits. In
addition to comparing against lifted MCMC, we also compare the theoretical upper
bound from Theorem~\ref{thm:mixing} against the two mixing times. This upper
bound only depends on the number of orbits, and does not depend on the
parameterization of the distribution.\footnote{For this example, there are 78 orbits.}
Orbit-jump MCMC converges to the true
distribution in both cases faster
than lifted MCMC, and the upper bound ensures that orbit-jump MCMC cannot
get stuck in low-probability orbits. Note however that lifted MCMC transitions
are less expensive to compute than orbit-jump MCMC transitions. We hope to
explore this practical tradeoff
between sample quality and the cost of drawing a sample in future work.

\section{RELATED WORK}
\textbf{Lifted inference}\quad
Existing exact lifted inference algorithms apply to relational models
~\citep{Getoor2007}. The tractability of exact lifted inference was studied by
\citet{NiepertAAAI14}, but their approach cannot be directly applied to factor
graphs. Approximate lifted inference can be applied to factor graphs, but
existing approaches do not provably mix quickly in the number of orbits~\citep{Niepert2012,
Niepert2013, Bui2013, VdBAAAI15a, MadanAMS18, Kersting2009, GogateJV12}. 

\textbf{Symmetry in constraint satisfaction and logic}\quad Some techniques for satisfiability and constraint satisfaction also exploit symmetry. The goal in that
context is to quickly select one of many symmetric candidate solutions, so a
key difference is that in our setting we must exhaustively explore the search space. \citet{Sabharwal2006} augments a SAT-solver with symmetry-aware
branching capabilities. Symmetry has also been exploited in integer-linear
programming \citep{Margot2010SymmetryII, Ostrowski2007OrbitalB,
  Margot2003ExploitingOI}.

\section{CONCLUSION \& FUTURE WORK}
In this paper we provided the first exact and approximate lifted inference
algorithms for factor graphs that provably scale in the number of orbits.
However, our methods are limited: there are tractable highly symmetric
distributions that still have too many orbits for our methods to be effective.
Existing lifted inference algorithms utilize independence to extract highly
symmetric sub-problems, which is an avenue that we can see for integrating
independence into this current approach. A further limitation of our approach is
that we exploit only symmetries on variables; additional
forms of symmetries, such as block symmetries, are beyond the scope of our
current algorithms~\citep{MadanAMS18}.


\section*{ACKNOWLEDGMENTS}
This work is partially supported by NSF grants \#IIS-1657613,
\#IIS-1633857, \#CCF-1837129, DARPA XAI grant \#N66001-17-2-4032, NEC Research,
a gift from Intel, and a gift from Facebook Research. The authors would like to
thank Tal Friedman, Pasha Khosravi, Jon Aytac, Philip Johnson-Freyd, Mathias
Niepert, and Anton Lykov for helpful discussions and feedback on drafts.

\bibliographystyle{abbrvnat}
\bibliography{bib}

\begin{thebibliography}{38}
\providecommand{\natexlab}[1]{#1}
\providecommand{\url}[1]{\texttt{#1}}
\expandafter\ifx\csname urlstyle\endcsname\relax
  \providecommand{\doi}[1]{doi: #1}\else
  \providecommand{\doi}{doi: \begingroup \urlstyle{rm}\Url}\fi

\bibitem[Artin(1998)]{Artin1998}
M.~Artin.
\newblock \emph{Algebra}.
\newblock Birkh{\"a}user, 1998.

\bibitem[Benhamou and Sais(1994)]{Benhamou1994}
B.~Benhamou and L.~Sais.
\newblock Tractability through symmetries in propositional calculus.
\newblock \emph{Journal of Automated Reasoning}, 12\penalty0 (1):\penalty0
  89--102, Feb 1994.

\bibitem[Bui et~al.(2013)Bui, Huynh, and Riedel]{Bui2013}
H.~H. Bui, T.~N. Huynh, and S.~Riedel.
\newblock Automorphism groups of graphical models and lifted variational
  inference.
\newblock In \emph{UAI}, pages 132--141, 2013.

\bibitem[Chavira and Darwiche(2005)]{ChaviraDarwiche2005}
M.~Chavira and A.~Darwiche.
\newblock Compiling bayesian networks with local structure.
\newblock In \emph{IJCAI}, pages 1306--1312, 2005.

\bibitem[Darwiche(2009)]{Darwiche2009}
A.~Darwiche.
\newblock \emph{Modeling and Reasoning with Bayesian Networks}.
\newblock Cambridge University Press, 2009.

\bibitem[De~Salvo~Braz et~al.(2005)De~Salvo~Braz, Amir, and
  Roth]{braz2005lifted}
R.~De~Salvo~Braz, E.~Amir, and D.~Roth.
\newblock Lifted first-order probabilistic inference.
\newblock In \emph{IJCAI}, pages 1319--1325, 2005.

\bibitem[Diaconis and Freedman(1980)]{Diaconis1980}
P.~Diaconis and D.~Freedman.
\newblock {De Finetti}'s generalizations of exchangeability.
\newblock In R.~C. Jeffrey, editor, \emph{Studies in Inductive Logic and
  Probability}, pages 2--233. Berkeley: University of California Press, 1980.

\bibitem[GAP()]{GAP4}
GAP.
\newblock \emph{{GAP -- Groups, Algorithms, and Programming, Version 4.10.0}}.
\newblock The GAP~Group, 2018.

\bibitem[Getoor and Taskar(2007)]{Getoor2007}
L.~Getoor and B.~Taskar.
\newblock \emph{Introduction to Statistical Relational Learning}.
\newblock The MIT Press, 2007.

\bibitem[Gogate and Domingos(2011)]{gogate2011probabilistic}
V.~Gogate and P.~Domingos.
\newblock Probabilistic theorem proving.
\newblock In \emph{UAI}, pages 256--265, 2011.

\bibitem[Gogate et~al.(2012)Gogate, Jha, and Venugopal]{GogateJV12}
V.~Gogate, A.~K. Jha, and D.~Venugopal.
\newblock Advances in lifted importance sampling.
\newblock In \emph{AAAI}, 2012.

\bibitem[Goldberg(2001)]{Goldberg2001}
L.~Goldberg.
\newblock Computation in permutation groups: Counting and randomly sampling
  orbits.
\newblock \emph{Surveys in Combinatorics}, pages 109--143, 2001.

\bibitem[Goldberg and Jerrum(2002)]{Goldberg2002}
L.~A. Goldberg and M.~Jerrum.
\newblock The {Burnside} process converges slowly.
\newblock \emph{Combinatorics, Probability and Computing}, 11\penalty0
  (1):\penalty0 21–34, 2002.

\bibitem[Jerrum(1993)]{Jerrum1992}
M.~Jerrum.
\newblock Uniform sampling modulo a group of symmetries using markov chain
  simulation.
\newblock \emph{DIMACS Series in Discrete Mathematics and Theoretical Computer
  Science}, pages 37--47, 1993.

\bibitem[Kersting(2012)]{Kersting2012}
K.~Kersting.
\newblock Lifted probabilistic inference.
\newblock In \emph{ECAI}, pages 33--38, 2012.

\bibitem[Kersting et~al.(2009)Kersting, Ahmadi, and Natarajan]{Kersting2009}
K.~Kersting, B.~Ahmadi, and S.~Natarajan.
\newblock Counting belief propagation.
\newblock In \emph{UAI}, pages 277--284, 2009.

\bibitem[Koller and Friedman(2009)]{Koller2009}
D.~Koller and N.~Friedman.
\newblock \emph{Probabilistic Graphical Models: Principles and Techniques}.
\newblock The MIT Press, 2009.

\bibitem[Levin and Peres(2017)]{levin2017markov}
D.~A. Levin and Y.~Peres.
\newblock \emph{Markov chains and mixing times}.
\newblock American Mathematical Society, 2017.

\bibitem[Liu(1996)]{liu1996metropolized}
J.~S. Liu.
\newblock Metropolized independent sampling with comparisons to rejection
  sampling and importance sampling.
\newblock \emph{Statistics and Computing}, 6\penalty0 (2):\penalty0 113--119,
  1996.

\bibitem[Madan et~al.(2018)Madan, Anand, Mausam, and Singla]{MadanAMS18}
G.~Madan, A.~Anand, Mausam, and P.~Singla.
\newblock Block-value symmetries in probabilistic graphical models.
\newblock In \emph{UAI}, pages 886--895, 2018.

\bibitem[Margot(2003)]{Margot2003ExploitingOI}
F.~Margot.
\newblock Exploiting orbits in symmetric {ILP}.
\newblock \emph{Mathematical Programming}, 98:\penalty0 3--21, 2003.

\bibitem[Margot(2010)]{Margot2010SymmetryII}
F.~Margot.
\newblock Symmetry in integer linear programming.
\newblock In \emph{50 Years of Integer Programming}, 2010.

\bibitem[McKay(1998)]{McKay1998}
B.~D. McKay.
\newblock Isomorph-free exhaustive generation.
\newblock \emph{Journal of Algorithms}, 26\penalty0 (2):\penalty0 306 -- 324,
  1998.

\bibitem[Mckay and Piperno(2014)]{Mckay2014}
B.~D. Mckay and A.~Piperno.
\newblock Practical graph isomorphism, {II}.
\newblock \emph{Journal of Symbolic Computation}, 60:\penalty0 94--112, 2014.

\bibitem[Niepert(2012)]{Niepert2012}
M.~Niepert.
\newblock Markov chains on orbits of permutation groups.
\newblock In \emph{UAI}, pages 624--633, 2012.

\bibitem[Niepert(2013)]{Niepert2013}
M.~Niepert.
\newblock Symmetry-aware marginal density estimation.
\newblock \emph{AAAI}, 2013.

\bibitem[Niepert and Van~den Broeck(2014)]{NiepertAAAI14}
M.~Niepert and G.~Van~den Broeck.
\newblock Tractability through exchangeability: {A} new perspective on
  efficient probabilistic inference.
\newblock In \emph{AAAI}, 2014.

\bibitem[Ostrowski et~al.(2007)Ostrowski, Linderoth, Rossi, and
  Smriglio]{Ostrowski2007OrbitalB}
J.~Ostrowski, J.~T. Linderoth, F.~Rossi, and S.~Smriglio.
\newblock Orbital branching.
\newblock In \emph{IPCO}, 2007.

\bibitem[Pak(2000)]{Pak00}
I.~Pak.
\newblock What do we know about the product replacement algorithm?
\newblock In \emph{Groups and Computation {III}}, pages 301--347, 2000.

\bibitem[Poole(2003)]{Poole2003}
D.~Poole.
\newblock First-order probabilistic inference.
\newblock In \emph{IJCAI}, pages 985--991, 2003.

\bibitem[Raz(2004)]{Raz2004}
R.~Raz.
\newblock Resolution lower bounds for the weak pigeonhole principle.
\newblock \emph{J. ACM}, 51\penalty0 (2):\penalty0 115--138, 2004.

\bibitem[Richardson and Domingos(2006)]{Richardson2006}
M.~Richardson and P.~Domingos.
\newblock Markov logic networks.
\newblock \emph{Machine Learning}, 62:\penalty0 107--136, 2006.

\bibitem[Roth(1996)]{Roth1996}
D.~Roth.
\newblock On the hardness of approximate reasoning.
\newblock \emph{Artificial Intelligence}, 82\penalty0 (1):\penalty0 273 -- 302,
  1996.

\bibitem[Sabharwal(2005)]{Sabharwal2006}
A.~Sabharwal.
\newblock Symchaff: A structure-aware satisfiability solver.
\newblock In \emph{AAAI}, volume~5, pages 467--474, 2005.

\bibitem[Seress(2003)]{Seress2003}
A.~Seress.
\newblock \emph{Permutation Group Algorithms}.
\newblock Cambridge Tracts in Mathematics. Cambridge University Press, 2003.

\bibitem[{The Sage Developers}(2018)]{sagemath}
{The Sage Developers}.
\newblock \emph{{S}ageMath, the {S}age {M}athematics {S}oftware {S}ystem
  ({V}ersion 8.5.0)}, 2018.
\newblock {\tt https://www.sagemath.org}.

\bibitem[Van~den Broeck(2013)]{VdBThesis13}
G.~Van~den Broeck.
\newblock \emph{Lifted inference and learning in statistical relational
  models}.
\newblock PhD thesis, 2013.

\bibitem[Van~den Broeck and Niepert(2015)]{VdBAAAI15a}
G.~Van~den Broeck and M.~Niepert.
\newblock Lifted probabilistic inference for asymmetric graphical models.
\newblock In \emph{AAAI}, 2015.

\end{thebibliography}

\appendix
\onecolumn
\section{Notation}
\label{sec:notation}
\begin{center}
 \begin{tabular}{cp{10cm}}
  \toprule
  Symbol & Meaning \\
  \midrule
   $\group$ & A group \\
   $\Omega$ & A set \\
   $x$ & An element of $\Omega$ \\
   $g \cdot x$ & Apply bijection $g \in \group$ to element $x \in \Omega$ \\
   $\graph$ & A graph \\
   $\graph_1 \cong \graph_2$ & $\graph_1$ is color-automorphic to $\graph_2$\\
   $\aut(\graph)$ & The automorphism group of $\graph$ \\
   $\bool$ & Set of Booleans, $\bool = \{\true, \false\}$ \\
   $\varX$ & A set of variables \\
   $\bool^\varX$ & The set of all possible assignments to $\varX$ \\
   $\varx$ & A variable assignment, $\varx \in \bool^\varX$ \\
   $\varx_1 \sim \varx_2$ & Orbit equivalence relation \\
   $d_{TV}$ & Total variation distance \\
   $\Omega / \group$ & Quotient of $\group$ acting on $\Omega$ \\
   $\sigma$ & Canonization function \\
   $\fgraph$ & Factor graph \\
   $\graph(\fgraph, \varx)$ & Colored graph assignment-encoding of $\varx$ \\
   $P_x^t$ & Distribution of Markov chain $P$ after $t$ steps beginning in state $x$.\\
  \bottomrule
\end{tabular}
\end{center}
\section{Proofs}
\subsection{Proof of Theorem~\ref{thm:coloredcorrespondence}}
\label{sec:coloredcorrespondenceproof}
  Let $\graph_1 = (V_1, E_1, C_1) = \graph(\fgraph, \varx)$ and $\graph_2 =
  (V_2, E_2, C_2) =  g \cdot
\graph(\fgraph, \varx)$. Assume $\varx \sim \varx'$. Then there exists an
element $g \in \aut(\graph)$ such that $g \cdot \varx = \varx'$. First we show
colors are preserved. By construction of the colored assignment encoding, for
any variable node $v \in \graph_1$, $\col(v,C_1) = \col(g \cdot v,C_2)$. The
colors of factor nodes are preserved because $\aut(\graph)$ by definition
preserves them. The fact that $g \in \aut(\graph)$ directly implies that vertex
neighborhoods are preserved. Then $\graph_1 \cong \graph_2$.

Assume $\graph_1 \cong \graph_2$; then there exists $g \in \aut(\graph)$ such
that $g \cdot \graph_1 = \graph_2$. By the construction of the colored encoding,
this $g$ also preserves the colors of the variable vertices, so $g \cdot \varx
= \varx'$.

\subsection{Proof of Theorem~\ref{thm:stabilizercorrespondence}}
\label{sec:stabilizercorrespondence}
  Let $\graph_1 = \graph(\fgraph, \varx)$ and $g \in
\stab_{\aut(\graph)}(\varx)$. Then $g \cdot \graph_1$ fixes all colors and
vertices, since $\graph_1$ is bipartite and $g$ fixes all variable nodes. Then,
$g \in \aut(\graph(\fgraph, \varx))$. Now let $g \in \aut(\graph(\fgraph,
\varx))$. By definition, $g$ fixes the colors of the variable nodes, so $g \in
\stab(\varx)$.

\subsection{Proof of Lemma~\ref{lem:pruning}}
\label{sec:pruningproof}
  Let $\graph_1 = \graph(\fgraph, \varx_1)$ and $\graph_2 = \graph(\fgraph,
\varx_2)$. Assume $x$ and $y$ are in the same orbit under $\aut(\graph(\fgraph,
\varx))$; then there exists $g \in \aut(\graph(\fgraph, \varx))$ such that $g
\cdot x = y$. There is only one vertex color that differs between $\graph_1$
and $\graph_2$: $x$ and $y$. Then, $g \cdot \graph_1 = \graph_2$, so
Theorem~\ref{thm:coloredcorrespondence} then shows $\varx_1 \sim \varx_2$.

\subsection{Proof of Theorem~\ref{thm:mixing}}
  
\label{sec:mixing}

The proof will proceed as follows. First, we will split up $\Pr$ into two
distributions: a between-orbit distribution, which describes the probability of
transitioning between two orbits, and a within-orbit distribution, which is
uniform. We will bound the total variation distance for these two
quantities, and combine these results to get a bound on the total variation
distance on the original distribution using the following lemma:

\begin{lemma}
  \label{lem:prod}
  Let $\mu(x,y)$ and $\nu(x,y)$ be two distributions on $X \times Y$. Let
$\mu_x(x) = \sum_y \mu(x,y)$, defined similarly for $\nu$. If for all $(x,y) \in
X \times Y$ it holds that $\Pr_\mu(y \mid x) = \Pr_\nu(y \mid x)$, then
$d_{TV}(\mu, \nu) = d_{TV}(\mu_x, \nu_x)$.
\end{lemma}
\begin{proof}
  \begin{align*}
    d_{TV}(\mu, \nu)
    =& \frac{1}{2}\sum_{x,y} \left| \Pr_\mu(x,y) - \Pr_\nu(x,y)  \right| \\
    =& \frac{1}{2}\sum_{x,y} \left| \Pr_\mu(y \mid x) \Pr_{\mu_x}(x) - \Pr_\nu(y \mid x) \Pr_{\nu_x}(x)  \right| & \text{Chain rule} \\
    =& \frac{1}{2}\sum_{x,y} \Pr_\mu(y \mid x) \times \left| \Pr_{\mu_x}(x) - \Pr_{\nu_x}(x) \right|
     & \text{Since } 0 \le \Pr_\mu(y \mid x) = \Pr_\nu(y \mid x) \le 1 \\
    =& \frac{1}{2} \sum_x \bigg( \left| \Pr_{\mu_x}(x) - \Pr_{\nu_x}(x) \right| \times
       \underbrace{\sum_y \Pr_\mu(y \mid x)}_{=1} \bigg)\\
    =& d_{TV}(\mu_x, \nu_x).
  \end{align*}
\end{proof}
Now we begin the main proof. Let $\Pr(x)$ be a $\group$-invariant distribution
on a set $\Omega$, and let $P_x^t(y)$ be the probability of transitioning from a
state $x$ to a state $y$ after $t$ steps under the orbit-jump proposal. We can
write $\Pr(x)$ as a product of a \emph{between-orbit} ($\Pr_B$) and
\emph{within-orbit} $(\Pr_W)$ distribution, where $\Pr_B$ is a distribution on
$\Omega / \group$ and $\Pr_W$ is a distribution on $\Omega$:
\begin{align}
  \Pr(x) = \underbrace{\Pr(x) \times |\orb(x)|}_{\Pr_B(\sigma(x))} \times
  \underbrace{\frac{1}{|\orb(x)|}}_{\Pr_W(x \mid \sigma(x))}
\end{align}
I.e., for some $o \in \Omega / \group$, for some $x \in \sigma^{-1}(o)$, $\Pr_B(o)
= \Pr(x) \times |\orb(x)|$.
Similarly, the distribution $P^t_x$ can be divided into a between-orbit and
within-orbit component. We define a new Markov chain $B$ \emph{between orbits}
that has the following transition rule from some initial state $\sigma(x) \in
\Omega / \group$:

\begin{enumerate}[noitemsep]
\item Sample $x' \sim \Pr_{\Omega / \group}$
\item Accept $\sigma(x')$ with probability $\frac{\Pr(x') \times
    |\orb(x')|}{\Pr(x) \times |\orb(x)|}$.
\end{enumerate}
Then, for some $y \in \Omega$ and $\hat{y} = \sigma(y)$,
\begin{align}
  P^t_x(y) = B^t_{\sigma(x)}(\hat{y}) \times \Pr_W(y \mid \hat{y}),
\end{align}
where we used the important fact that the orbit-jump proposal that defines
$P^t_x$ is uniform
within orbits. Now we can rewrite the total variation distance that we wish to
upper-bound:
\begin{align}
  \begin{split}
    d_{TV}\big(P_x^t(y), \Pr(y)\big) =
    d_{TV}\big(B^t_{\sigma(x)}(\hat{y}) \times \Pr_W(y \mid \hat{y}), 
    \Pr_B(\hat{y}) \times \Pr_W(y \mid \hat{y})\big)
  \end{split}
\end{align}

Now using Lemma~\ref{lem:prod} we can simplify the bound on the total variation
distance to be the total variation distance of the between-orbit distributions:
\begin{align}
    d_{TV}(P_x^t(y), \Pr(y)) = d_{TV}(B^t_{\sigma(x)}(\hat{y}), \Pr_B(\hat{y})).
\end{align}

Now, our goal is to upper-bound $d_{TV}(B^t_{\sigma(x)}, \Pr_B)$. To do this we will use a
standard \emph{coupling argument}. A \emph{coupling} is a way to run two copies of a
Markov chain $P$ at the same time with the following properties:
\begin{enumerate}[noitemsep]
\item Both copies in isolation evolve according to $P$;
\item If both copies are in the same state, they remain in the same state.
\end{enumerate}

Two coupled chains can be used to acquire upper-bounds on the total variation
distance of a Markov chain by upper-bounding the probability that a Markov chain
starting from two initial distributions -- one in its stationary distribution
and the other in an arbitrary location -- will \emph{coalesce} into the same
state:

\begin{lemma}[\citep{levin2017markov} Theorem 5.4]
  Let $P$ be a transition matrix on state-space $\Omega$ with stationary
distribution $\pi$. Let $\{(X_t, Y_t)\}$ be coupled chains that evolve according
to $P$ of length $t$, starting from an initial state $x \in
\Omega$ and $y \sim \pi$. Then,
  \begin{align}
    d_{TV}(P^t_x, \pi) \le \Pr(X_t \ne Y_t).
  \end{align}
  \label{lem:coupling}
\end{lemma}

Now we define the
coupled chains $\{(X_t, Y_t)\}$.
Let $X_0 \in \Omega/\group$ be an arbitrarily chosen initial element, and let
$Y_0 \sim \Pr_B$ be an element chosen according to $\Pr_B$.
At each time step $t$, choose a state $o \in
\Omega/\group$ uniformly at random. Then, \emph{both} chains attempt to
transition to $o$, using the standard metropolis correction criteria to decide
whether or not to accept $o$. In order to guarantee coalescence, if both chains are in the same state, then we
define them to accept or reject a new state together. Intuitively, these two chains simulate the Markov
chain $B$ starting from different initial states, where they both share a common
source of randomness. Then by Lemma~\ref{lem:coupling},
\begin{align}
  d_{TV}(B^t_{X_0}, \Pr_B) \le \Pr(X_t \ne Y_t).
\end{align}

This probability can be upper bounded as follows. There exists a (possibly non-unique)
maximum probability state $M \in \Omega / \group$:
\begin{align*}
  M = \sigma\Big(\argmax_{x}~ \Pr(x) \times |\orb(x)|\Big).
\end{align*}
If both Markov chains
uniformly choose $M$ to transition to, then by the Metropolis rule they will
both accept and thus coalesce. Since the proposal is uniform, $\Pr(X_t \ne Y_t)$
is upper-bounded by the probability of not transitioning to $M$ after $t$ steps,
so:
\begin{align}
  \Pr(X_t \ne Y_t) \le \left( \frac{|\Omega / \group| - 1}{|\Omega / \group|} \right)^t,
\end{align}
which gives the first bound in the theorem. This quantity can be upper bounded
by a parameter $\varepsilon > 0$ representing the chosen error tolerance.
Solving for $t$:
\begin{align*}
  t \ge \log(\varepsilon) \times \left[ \log \left( \frac{|\Omega / \group| - 1}{|\Omega / \group|} \right)\right]^{-1}
\end{align*}
Using the identity:
\begin{align*}
  \log \left( \frac{x-1}{x} \right) = - \left( \frac{1}{x} + \frac{1}{2x^2} + \cdots \right),
\end{align*}
we then have that:
\begin{align}
  t \ge \log(\varepsilon^{-1}) \times |\Omega / \group|  \ge
  \log(\varepsilon^{-1}) \times \left( \frac{1}{|\Omega / \group|} + \frac{1}{2|\Omega / \group|^2} + \cdots \right)^{-1},
\end{align}
which gives the second bound and concludes the proof.

\end{document}